\documentclass[twoside,11pt]{article}
\usepackage[preprint]{jmlr2e}
\usepackage{lastpage}
\ShortHeadings{Calibration-then-Calculation: A Variance Reduced Metric
Framework}{Fan, Si, Song, and Zhang}
\firstpageno{1}

\usepackage[utf8]{inputenc} %
\usepackage[T1]{fontenc}    %
\usepackage{hyperref}       %
\usepackage{url}            %
\usepackage{booktabs}       %
\usepackage{amsfonts}       %
\usepackage{nicefrac}       %
\usepackage{microtype}      %
\usepackage{xcolor}         %
\usepackage{wrapfig}
\usepackage{tikz}
\usetikzlibrary{arrows,positioning}

\usepackage{nccmath}
\usepackage{amsmath}
\usepackage{amssymb}
\usepackage{bm}
\usepackage{bbm}

\usepackage{stackengine}
\usepackage{subcaption}
\usepackage{paralist}
\usepackage{tabularx}
\usepackage{scalerel,graphicx,xparse,textcomp}
\usepackage{mathtools}

\usepackage{etoolbox}
\usepackage[capitalize,noabbrev]{cleveref}

\usepackage{enumitem}
\usepackage{multirow}
\usepackage{changepage} %
\usepackage{adjustbox} %
\usepackage{makecell}
\usepackage{lipsum}
\usepackage{cancel}
\usepackage{algorithm}
\usepackage{algorithmic}
\usepackage{multicol}
\usepackage{pifont}
\usepackage{stmaryrd}
\usepackage{trimclip}
\usepackage{tikzducks}
\usepackage[normalem]{ulem}
\usepackage{setspace}

\usepackage{float}
\usepackage{verbatim}
\usepackage{array}
\usepackage{xspace}

\graphicspath{{pic/}}

\begin{document}
\title{Calibration-then-Calculation: A Variance Reduced Metric Framework in Deep Click-Through Rate Prediction Models}

   
\author{
Yewen Fan$^{1}$ \quad Nian Si$^{3}$ \quad
Xiangchen Song$^{1}$ \quad Kun Zhang$^{12}$
\\ \\
$^1$Carnegie Mellon University\\
$^2$Mohamed bin Zayed University of Artificial Intelligence\\
$^3$University of Chicago Booth School of Business\\
 }

\editor{My editor}
\maketitle

\begin{abstract}
The adoption of deep learning across various fields has been extensive, yet there is a lack of focus on evaluating the performance of deep learning pipelines. Typically, with the increased use of large datasets and complex models, the training process is run only once and the result is compared to previous benchmarks. This practice can lead to imprecise comparisons due to the variance in neural network evaluation metrics, which stems from the inherent randomness in the training process. Traditional solutions, such as running the training process multiple times, are often infeasible due to computational constraints. In this paper, we introduce a novel metric framework, the Calibrated Loss Metric, designed to address this issue by reducing the variance present in its conventional counterpart. Consequently, this new metric enhances the accuracy in detecting effective modeling improvements. Our approach is substantiated by theoretical justifications and extensive experimental validations within the context of Deep Click-Through Rate Prediction Models.

\end{abstract}


\section{Introduction}\label{sec:introduction}

The progress in machine learning is largely influenced by experimental outcomes, particularly in the era of deep learning. Researchers often evaluate the performance of new methods by comparing them to previous benchmark results to demonstrate the superiority of new methods. However, it is well known that the performance of deep learning models can vary greatly, even when using the same pipeline \citep{seedallyouneed}, where, in this work, we define the pipeline broadly, which includes but is not limited to the selection of feature sets,  model architectures, optimization algorithms, initialization schemes, and hyper-parameters. Two identical pipelines may produce substantially different validation metrics due to factors like random initialization, shuffling, and optimization noise. This variability makes it difficult to precisely compare modeling improvements over previous baselines. Even significant engineering efforts may only lead to small measured gains within the noise margin.

In fact, it has been shown that by selecting a fortunate random initialization seed, one can achieve a model performance that is significantly better than average \citep{seedallyouneed}. This difference can be substantial enough to be used as a strong argument for publications in selective venues \citep{seedallyouneed}.

This issue is exacerbated in industry where the production model performance is hard to get improvement, while there are hundreds of machine learning engineers working on the same model at the same time. The performance gain of a modeling proposal is usually small and within the metric variance, making it difficult to judge the effectiveness of the modeling proposal.

To address this issue, a common approach is to run the training pipeline multiple times and report the average, standard deviation, minimum, and maximum performance scores \citep{seedallyouneed}. However, with the rise of large training data and big models, this approach is not always practical due to limited computational resources \citep{bouthillier2020survey}.

In this work, we present a new perspective focusing on metric design. The key insight is that we can reduce the variance of evaluation metrics themselves to enable more accurate comparisons between models. We propose a metric framework called Calibrated Loss Metric that exhibits lower variance than standard losses like Log Loss. Our method works by correcting inherent biases before finalizing metric calculations on a holdout set. 

We provide theoretical justifications for our metric under simplified linear models. More importantly, we demonstrate the effectiveness empirically across diverse deep learning pipelines for CTR prediction. Experiments on real-world advertising data show that our Calibrated Log Loss Metric reliably detects superior models compared to alternatives. We further validate the benefits under different training configurations. Our results suggest that our method can significantly improve the efficiency of benchmarking and optimizing neural recommendation models.

In summary, this paper makes the following contributions:
\begin{itemize}
\item We formulate the deep learning pipeline evaluation problem and propose to tackle it by designing new metrics. 

\item We propose a new metric framework, Calibrated Loss Metric, which can mitigate the above deep learning pipeline evaluation issue. 

\item We conduct extensive experiments to demonstrate the effectiveness of the proposed metric, using synthetic dataset and ads click dataset. 

\item We provide theoretical guarantees under linear regression setting that the proposed metric has a smaller variance than its vanilla counterpart.
\end{itemize}

The remainder of the paper is organized as follows. We first present the problem formulation and metric design. We then provide theoretical justifications and experiment results. Finally, we survey related literature and offer concluding remarks.

\section{Preliminaries and Problem Setting}

In this work, we are examining the standard supervised learning setting, where we assume that the training data and test data are randomly drawn from an unknown distribution in an i.i.d. manner, denoted as $\mathcal{D}$. 


Our goal is to develop a good pipeline that maps from a training distribution to a possibly random model, $h \in \mathcal{H}$, that generalizes well during the test time. As we mentioned in Introduction, the pipeline includes the whole procedures of training a model, including the selection of  model architectures, optimization algorithms, initialization schemes, and hyper-parameters.   Model performance is evaluated by a metric, $e$, and thus the expected performance of a model $h$ during the test time is 
\begin{equation}
    R_e(h) = \mathbb{E}_\mathcal{D}[e(h(X),Y)|h].
\end{equation}

In practice, $R_e(h)$ is estimated by finite-sample average on the test dataset $\mathcal{\hat{D}}_{\mathrm{test}}$. That is, 
\begin{equation}
\hat{R}_e(h, \mathcal{\hat{D}}_{\mathrm{test}}) = \frac{1}{|\mathcal{\hat{D}}_{\mathrm{test}}|}\sum_{(x, y) \in \mathcal{\hat{D}}_{\mathrm{test}}} e(h(x), y).
\end{equation}

It's worth noting that the expected risk, $R_e(h)$, is a random variable, as $h$ is random and it depends on a specific model that is produced by the underlying deep learning pipeline. The output model is random due to data randomness from the sample collection and intrinsic randomness during the training process in the deep learning pipeline, such as data order and randomness from stochastic gradient descent. Therefore, a proper evaluation and comparison of different deep learning pipelines should take into account the distribution of $R_e(h)$ \citep{bouthillier2021accounting}. It's also important to note that the term "deep learning pipeline" in this context is general, as we consider different model configurations (e.g. different model hyperparameters) as different "deep learning pipelines", even though they may belong to the same model class.

To compare the performance of different deep learning pipelines, we should calculate the expected risk ${R}_e(h)$ for each pipeline. As mentioned above, such expected risk is a random variable w.r.t $h$, then we should compare the distribution of ${R}_e(h)$, specifically we use the probability of the event that the expected risk ${R}_e(h)$ for one pipeline is larger or smaller than the other to quantify the pair-wise performance comparison among different pipelines. 
\begin{definition}
\textit{(Better pipeline)} For any two pipelines $A$ and $B$, we say that pipeline $A$ is better than pipeline $B$ with respect to metric $e$ if and only if the probability that pipeline $A$ produces a better model (i.e. smaller risk), measured by the metric $e$, is larger than $0.5$. This is represented by the inequality: 
\begin{equation}
\label{ineq:def:better}
    P\left({R}_e(h_A) < {R}_e(h_B)\right) > 0.5,
\end{equation}
where $h_A$ and $h_B$ are random variables representing the output models produced by pipeline $A$ and $B$ respectively.
\end{definition}

Our objective is to compare the performance of two pipelines, $A$ and $B$, with respect to the metric $e$ by running the training pipeline only once. Ideally, to estimate $P({R}_e(h_A) < {R}_e(h_B))$, we could use the Monte Carlo method, but this requires a huge amount of computational resources which is not realistic.
In this work, we propose a different approach: we aim to come up with an alternative metric $e_1$ with the following properties:
\begin{enumerate}
    \item \textit{Roughly same mean} $$\mathbb{E}\left[{R}_{e_1}(h)\right] \approx \mathbb{E}\left[{R}_e(h)\right];$$
    \item \textit{Strictly small variance} $$var\left({R}_{e_1}(h)\right) < var\left({R}_e(h)\right),$$
\end{enumerate}
where the randomness is from the pipeline producing $h$.
As a result, the new metric is able to compare the performance of pipelines $A$ and $B$ with limited computational resources more accurately.
\begin{definition}
\textit{(Better alternative metric)} Assuming that pipeline $A$ is better than $B$ with respect to the metric $e$ (i.e. pipeline $A$ is more likely to produce a better model than pipeline $B$ in the ground truth if measured by metric $e$), we say that a metric $e_1$ is better than $e_2$ with respect to metric $e$ if and only if the probability that pipeline $A$ produces a better model than pipeline $B$ measured by metric $e_1$ is greater than the probability measured by metric $e_2$. This is represented by the inequality: 
\begin{equation}
\label{ineq:comparison}
P({R}_{e_1}(h_A) < {R}_{e_1}(h_B)) > 
P({R}_{e_2}(h_A) < {R}_{e_2}(h_B)).
\end{equation}
\end{definition}
In other words, using metric $e_1$ is more likely to accurately detect that pipeline $A$ is better than pipeline $B$, which aligns with the ground truth. Here, we allow for a general form of the risk function, which may not admit the expectation form; i.e.,  ${R}_{e_1}(h)$ may not necessarily have the form $\mathbb{E}_\mathcal{D}[e_1(h(X),Y)]$.

\begin{definition}

\textit{(Metric accuracy)} Here, we assume without loss of generality that pipeline $A$ is better than $B$ with respect to the metric $e$. We define the accuracy of a metric $\bar{e}$ with respect to metric $e$ and pipeline A and B as:
 \begin{equation}
    \mathrm{Acc}(\bar{e}) \triangleq P({R}_{\bar{e}}(h_A) < {R}_{\bar{e}}(h_B)).
\end{equation}
\end{definition}

Our goal is to find a metric $\bar{e}$ that has higher accuracy than the original metric $e$ for a wide range of pipelines $A$ and $B$. In the next section, we will present a new metric framework, Calibrated Loss Metric. In the context of Deep Click-Through Rate Prediction Models, a special instance of Calibrated Loss Metric, Calibrated Log Loss Metric achieves higher accuracy than the vanilla Log Loss Metric. The intuition is that the bias in the function $h$ is always volatile and carries on a large amount of randomness.  Calibrating the bias will usually not change the comparison between two pipelines but can reduce the randomness. In Section \ref{section:theory}, we will present a theoretical analysis that justifies this intuition by showing that our new metric framework has a smaller variance under the linear regression setting. Through extensive experiments in Section \ref{section:experiment}, we will show that Calibrated Log Loss Metric achieves higher accuracy than Log Loss Metric for a wide range of pipelines, including those with different batch sizes, number of features, model architectures, regularization weights, model sizes, etc.

\section{Calibrated Loss Metric Framework}

\textbf{Calibrated Log Loss Metric (Binary Classification)} In the field of Deep Click-Through Rate Prediction Models, it is common for models to overfit when trained for more than one epoch \citep{zhou2018deep, zhang2022towards}. As a result, models are often only trained for a single epoch in practice \citep{zhang2022towards}, leaving it uncertain if the model has been fully optimized. This leads to the volatility of the bias term in the final layer of neural networks, creating additional randomness. To address this issue, we propose the following risk function:
\begin{align*}
 R_{e_1}(h) = \min_c \mathbb{E}_{\mathcal{D}} [Y\log(h^c(X))+(1-Y)\log(1-h^c(X))|h],
\end{align*}
where
\begin{align*}
 h^c(X)= ({1 + e^{-\mathrm{logit}(h(X)) + c}})^{-1},
\end{align*}
\begin{align*}
\mathrm{logit}(p) = \log(\frac{p}{1-p}).
\end{align*}

To execute the aforementioned procedure on a finite-sample test set, we split the test data $\mathcal{\hat{D}}_{\mathrm{test}}$ into two parts: a test-validation dataset $\mathcal{\hat{D}}_{\mathrm{val-test}}$ and a remaining test dataset $\mathcal{\hat{D}}_{\mathrm{remaining-test}}$. By using the test-validation dataset $\mathcal{\hat{D}}_{\mathrm{val-test}}$, we are able to correct the bias term, and then calculate the Log Loss Metric using $\mathcal{\hat{D}}_{\mathrm{remaining-test}}$ with the bias-adjusted predictions. This is referred to as Calibrated Log Loss Metric. The calculation procedure is outlined in Algorithm \ref{algo:cali_free}.

Mathematically speaking, we define bias-adjusted predictions as: $q_i=h^c(x_i)$ for $x_i$ in the test set
where $c$ is the bias-adjusted term we are optimizing. 

To optimize $c$, the following optimization program is solved, which is the log loss between bias-adjusted predictions $q_i$ and labels $y_i$: 
\begin{equation}
\label{eq:optimization_problem}
    \min_{c}\left\{ -\sum_{\substack{ (x,y) \in \\ \mathcal{\hat{D}}_{\mathrm{val-test}}}} \left(y \log(h^c(x)) + (1-y) \log(1-h^c(x)) \right) \right\}.
\end{equation}

It can be easily shown that, after optimization, the bias-adjusted predictions $q_i$ are well-calibrated in $\mathcal{\hat{D}}_{\mathrm{val-test}}$, meaning that 
$$\sum_{i \in \mathcal{\hat{D}}_{\mathrm{val-test}}} q_i = \sum_{i \in \mathcal{\hat{D}}_{\mathrm{val-test}}} y_i.$$

 Let $c^*$ be the minimizer of the optimization problem (\ref{eq:optimization_problem}). The final risk and metrics are
 \begin{align*}
 \hat{R}_{e_1}(h, \mathcal{\hat{D}}_{\mathrm{test}}) = \frac{1}{|\mathcal{\hat{D}}_{\mathrm{remaining-test}}|}\sum_{\substack{(x, y) \in  \\ \mathcal{\hat{D}}_{\mathrm{remaining-test}}}} e_1(h(x), y),
 \end{align*}
and 
$$
e_1(h(x), y)=y \log(h^{c^*}(x)) + (1-y) \log(1-h^{c^*}(x)).  
$$

\textbf{Explanations} The optimization problem (\ref{eq:optimization_problem}) corrects the bias term of original predictions $h(x)$ using test-validation dataset $\mathcal{\hat{D}}_{\mathrm{val-test}}$. The bias-adjusted predictions $h^{c^*}(x)$ is guaranteed to be well-calibrated in $\mathcal{\hat{D}}_{\mathrm{val-test}}$, hence the name Calibrated Log Loss Metric.

\begin{algorithm}[ht]
\caption{Calculate Calibrated Log Loss Metric}
\label{algo:cali_free}

\begin{algorithmic}[1]
\STATE  \textbf{Input:} Model $M$, labeled test data $\mathcal{\hat{D}}_\mathrm{test}$.
\STATE \textbf{Output:} Calibrated Log Loss Metric: $\hat{R}_{e_1}(h, \mathcal{\hat{D}}_{\mathrm{test}})$.

\STATE Partition $\mathcal{\hat{D}}_\mathrm{test}$ into $\mathcal{\hat{D}}_\mathrm{val-test}$ and $\mathcal{\hat{D}}_\mathrm{remaining-test}$.
\STATE Compute model predictions on $\mathcal{\hat{D}}_\mathrm{val-test}$ and $\mathcal{\hat{D}}_\mathrm{remaining-test}$, and obtain the model predictions $p^\mathrm{val-test}_i$ and $p^\mathrm{remaining-test}_i$.
\STATE Solve the optimization problem (\ref{eq:optimization_problem}) using $p^\mathrm{val-test}_i$ and $y^\mathrm{val-test}_i$ and obtain the learned bias term $c^*$.
\STATE Calculate bias-adjusted predictions $q^\mathrm{remaining-test}_i$ using formula $q_i=h^{c^*}(x_i)$ . 
\STATE Calculate the Calibrated Log Loss Metric $ \hat{R}_{e_1}(h, \mathcal{\hat{D}}_{\mathrm{test}})$ as the Log Loss Metric of $q^\mathrm{remaining-test}_i$ and $y^\mathrm{remaining-test}_i$.

\end{algorithmic}
\end{algorithm}

\textbf{Generalization to Quadratic Loss Metric}  Calibrated Quadratic Loss Metric is calculated in a similar manner as Calibrated Log Loss Metric, i.e. first perform calibration on $\mathcal{\hat{D}}_{\mathrm{val-test}}$ and calculate bias-adjusted predictions on $\mathcal{\hat{D}}_{\mathrm{remaining-test}}$. Here, we define the Quadratic Loss Metric and Calibrated Quadratic Loss Metric:
\[
e(h(x),y)=(y-h(x))^2, \text{ and}
\]
\begin{equation*}
e_1(h(x),y)=(y-h(x)-(\mathbb{E}_\mathcal{D}[Y]-\mathbb{E}_\mathcal{D}[h(X)|h]))^2.
\end{equation*}

\section{Theory on Linear Regression}
\label{section:theory}

In this section, we provide theoretical justification that our new metric has a smaller variance than its vanilla counterpart under Linear Regression setting, where the randomness only comes from the data randomness. We choose to provide a theoretical guarantee under Linear Regression due to its simplicity. We empirically verify our method's performance under Logistic Regression and Neural Networks in the next section. Note that in Linear Regression, Quadratic Loss Metric is used instead of Log Loss Metric. As a result, in our theory, we compare the variance of Calibrated Quadratic Loss Metric with vanilla Quadratic Loss Metric.

\begin{theorem}
\label{theorem}
Suppose that the features $X\in \mathbb{R}^d$ and the label $Y$ are distributed jointly Gaussian. We consider linear regression $h(x) = \beta ^\top x +\alpha $. Let $\hat{\beta}_n$ be the coefficient learned from the training data with sample size $n$. 
Then, we have 
\begin{equation*}
\left(1+\frac{1}{n}\right)\mathbb{E}[e_1(h(X),Y)|\hat{\beta}_n ]= \mathbb{E}[e(h(X),Y)|\hat{\beta}_n],
\end{equation*}
where the expectation is taken over the randomness over both the training and test samples.
\end{theorem}
Let $\hat{\alpha}_n$  be the learned intercept. Note that the original risk and the calibrated risk are 
\begin{align*}
R_e(h) &=  \mathbb{E}[e(h(X),Y)|\hat{\beta}_n,\hat{\alpha}_n], \text{ and} \\
R_{e_1}(h) &=  \mathbb{E}[e_1(h(X),Y)|\hat{\beta}_n,\hat{\alpha}_n]= \mathbb{E}[e_1(h(X),Y)|\hat{\beta}_n].
\end{align*}
Therefore, Theorem \ref{theorem} implies that $$(1+\frac{1}{n}) \mathbb{E}[R_{e_1}(h)]=\mathbb{E}[R_{e}(h)].$$ 
Furthermore, to make $e$ and $e_1$ comparable, we should scale $e_1$ to $(1+\frac{1}{n})e_1$. We demonstrate that after scaling, $(1+\frac{1}{n})R_{e_1}(h)$ has a smaller variance than $R_{e}(h)$ in the next corollary. In practice, as $(1+\frac{1}{n})$ is a constant as long as the training sample size is fixed, we can directly compare two pipelines using $R_{e_1}(h)$.

\begin{corollary}
\label{corollary}
Suppose that $h_1(x)$ and  $h_2(x)$ are two different learned linear functions in different feature sets. Then, we have
\begin{eqnarray}\mathbb{E}[R_e(h_1)] =\mathbb{E}[R_e(h_2)] 
 \Leftrightarrow \mathbb{E}[R_{e_1}(h_1)] =\mathbb{E}[R_{e_1}(h_2)]
\end{eqnarray}
and  
$$var\left(\left(1+\frac{1}{n}\right)R_{e_1}(h)\right)<var(R_{e}(h))$$ for any $h$ learned from linear regression.
\end{corollary}

Corollary \ref{corollary} indicates that Calibrated Quadratic Loss Metric has a smaller variance than vanilla Quadratic Loss Metric without changing the mean after appropriate scaling. Note that smaller variance and higher accuracy (Inequality \ref{ineq:comparison}) are highly correlated under mild conditions, but smaller variance alone does not guarantee higher accuracy. In the next section, we will empirically demonstrate that the new metric has a smaller variance and achieves higher accuracy. All proofs can be found in Appendix \ref{appendix:proof}.

\section{Experiment Results}
\label{exp}

\label{section:experiment}

\subsection{Estimation of Accuracy}

Recall that accuracy of a metric $\bar{e}$ is defined as:
\begin{equation*}
     \mathrm{Acc}(\bar{e}) \triangleq P({R}_{\bar{e}}(h_A) < {R}_{\bar{e}}(h_B)).
\end{equation*}
To get an estimation of $\mathrm{Acc}(\bar{e})$, we run pipelines $A$ and $B$ for $m$ times, obtaining models $h_{A_i}$ and $h_{B_i}$ for $i \in [m]$. $\mathrm{Acc}(\bar{e})$ can be estimated as:
\begin{equation}
    \mathrm{\widehat{Acc}}(\bar{e}) = \frac{1}{m^2} \sum_{(i, j)} \mathbbm{1}{(\hat{R}_{\bar{e}}(h_{A_i}, \mathcal{\hat{D}}_{\mathrm{test}}) < \hat{R}_{\bar{e}}(h_{B_j}, \mathcal{\hat{D}}_{\mathrm{test}}))}
\end{equation}
$\mathrm{\widehat{Acc}}(\bar{e})$ is an unbiased estimator of $\mathrm{Acc}(\bar{e})$, and in the experiments below, we report $\mathrm{\widehat{Acc}}(\bar{e})$ as our accuracy metric.  In all the tables in this section, without loss of generality, we write the tables as pipeline A is better than pipeline B in the sense of $P({R}_e(h_A) < {R}_e(h_B)) > 0.5$.

\subsection{Synthetic Data}
In Appendix \ref{appendix:experiments}, we consider a linear regression model to give empirical evidence to support our theory. We further consider logistic regression model to demonstrate the effectiveness of Calibrated Log Loss Metric in synthetic data setting. All the details and results can be found in Appendix \ref{appendix:experiments}.

\subsection{Avazu CTR Prediction dataset}

\textbf{Dataset} 
The Avazu CTR Prediction dataset \footnote{https://www.kaggle.com/c/avazu-ctr-prediction} is a common benchmark dataset for CTR predictions.
Due to computational constraints in our experiments, we use the first 10 million samples, shuffle the dataset randomly, and split the whole dataset into $80\%$ $\mathcal{\hat{D}}_\mathrm{train}$, $2\%$ $\mathcal{\hat{D}}_\mathrm{val-test}$ and $18\%$ $\mathcal{\hat{D}}_\mathrm{remaining-test}$.

\textbf{Metrics}
We compare the accuracy of Calibrated Log Loss Metric and Log Loss metric. We compute Calibrated Log Loss Metric as outlined in Algorithm \ref{algo:cali_free}. To make fair comparison, we compute Log Loss metric on $\mathcal{\hat{D}}_\mathrm{test}$ = $\mathcal{\hat{D}}_\mathrm{val-test}$ + $\mathcal{\hat{D}}_\mathrm{remaining-test}$.

\textbf{Base Model}
We use the xDeepFM model \citep{lian2018xdeepfm} open sourced in \citet{shen2017deepctr} as our base model. We primarily conduct experiments using xDeepFM models, including hyperparameter related experiments and feature related experiments. To demonstrate our new metric can also handle comparisons between different model architectures, we also conduct experiments using DCN \citep{wang2017deep}, DeepFM \citep{guo2017deepfm}, FNN \citep{zhang2016deep}, and DCNMix \citep{wang2021dcn}.

\textbf{Experiment Details} 
We consider neural networks with different architectures, different training methods, different hyper-parameters, and different levels of regularization as different pipelines. Such comparisons represent common practices for research and development in both industry and academia. 
For each pipeline, we train the model $60$ times with different initialization seeds and data orders to calculate $\mathrm{\widehat{Acc}}(\bar{e})$. Note that we use "Log Loss Metric" as our ground truth metric to determine the performance rank of different pipelines. Due to computational constraints, we cannot afford to run the experiments for multiple rounds. Instead, we run the experiments for one round and report accuracy. Note that in the neural network experiments, we do not re-sample the training data each time, as there is intrinsic randomness in the neural network training process. This is the main difference from the Linear Regression and Logistic Regression experiments.

\textbf{Pipelines with Different Number of Features} 
In this set of experiments, for pipeline $A$, we use all the features available. For pipeline $B$, we remove some informative features. We tested the removal of $6$ dense features and $1$ sparse features respectively. 

\begin{table}[!ht]
    \centering
    \caption{Accuracy of Log Loss Metric (LL) and Calibrated Log Loss Metric (CLL) under Neural Networks (features)}
    \begin{tabular}{llcc}
    \toprule
        Pipeline A & Pipeline B & LL Acc & CLL Acc \\
    \midrule
        Baseline & remove dense & 81.8\% & 88.8\% \\
        Baseline & remove sparse & 78.6\% & 85.9\% \\
    \bottomrule
    \end{tabular}
    \label{exp:nn_features}
\end{table}

\begin{table}[!ht]
    \centering
    \caption{Mean and Standard Deviation of Log Loss Metric (LL) and Calibrated Log Loss Metric (CLL) under Neural Networks (features)}
        \begin{tabular}{lcccc}
        \toprule
            Pipeline & LL Mean & CLL Mean & LL Std & CLL Std \\
        \midrule
            remove dense & 0.37408 & 0.37403 & $4.7 \times 10^{-4}$ & $3.8 \times 10^{-4}$ \\
            remove sparse & 0.37404 & 0.37398 & $5.0 \times 10^{-4}$ & $4.2 \times 10^{-4}$ \\
        \bottomrule
        \end{tabular}
    \label{exp:nn_arch_features}
\end{table}

From the result in Table \ref{exp:nn_features}, we can clearly see that Calibrated Log Loss Metric has a higher accuracy, indicating its effectiveness when comparing the performance of pipelines with different features. 

From the result in Table \ref{exp:nn_arch_features}, we can see that Calibrated Log Loss Metric has a smaller standard deviation (16\% - 19\% smaller) while the mean of Log Loss Metric and Calibrated Log Loss Metric is almost on par (within 0.02\% difference).

\textbf{Pipelines with Different Model Architectures} 
In this set of experiments, we aim to find out whether the new metric is able to detect modeling improvement from architecture changes. We tested a variety of different model architectures, including DCN \citep{wang2017deep}, DeepFM \citep{guo2017deepfm}, FNN \citep{zhang2016deep}, and DCNMix \citep{wang2021dcn}.

\begin{table}[!ht]
    \centering
    \caption{Accuracy of Log Loss Metric (LL) and Calibrated Log Loss Metric (CLL) under Neural Networks (model architectures)}
    \begin{tabular}{llcc}
    \toprule
        Pipeline A & Pipeline B & LL Acc & CLL Acc \\
        \midrule
        DCN & DCNMix & 64.4\% & 71.5\% \\ 
        DeepFM & DCN & 77.2\% & 83.9\% \\
        DeepFM & FNN & 76.9\% & 79.9\% \\ 
        FNN & DCNMix & 61.5\% & 72.0\% \\ 
        DeepFM & DCNMix & 84.8\% & 93.4\% \\ \bottomrule
    \end{tabular}
    \label{exp:nn_arch}
\end{table}
\begin{table}[ht]
    \centering
    \caption{Mean and Standard Deviation of Log Loss Metric (LL) and Calibrated Log Loss Metric (CLL) under Neural Networks (model architectures)}
        \begin{tabular}{lcccc}
        \toprule
            Pipeline & LL Mean & CLL Mean & LL Std & CLL Std \\
        \midrule
            DCN & 0.38021 & 0.38011 & $4.4 \times 10^{-4}$ & $3.3 \times 10^{-4}$ \\ 
            DeepFM & 0.37971 & 0.3796 & $5.9 \times 10^{-4}$ & $3.7 \times 10^{-4}$ \\ 
            FNN & 0.38029 & 0.38006 & $6.4 \times 10^{-4}$ & $4.0 \times 10^{-4}$ \\ 
            DCNMix & 0.38046 & 0.38037 & $4.8 \times 10^{-4}$ & $3.4 \times 10^{-4}$ \\
        \bottomrule
        \end{tabular}
    \label{exp:nn_arch_hyperparameters}
\end{table}
From the result in Table \ref{exp:nn_arch}, we can clearly see that Calibrated Log Loss Metric has higher accuracy, again indicating its effectiveness when comparing the performance of pipelines with different model architectures. In Table \ref{exp:nn_arch_hyperparameters}, we report the mean and standard deviation of Log Loss Metric and Calibrated Log Loss Metric, consistent with previous results.


\textbf{Pipelines with Different Model Hyperparameters}
In this set of experiments, we compare pipelines with different model hyperparameters, including neural network layer size, Batch Normalization (BN) \citep{ioffe2015batch}, Dropout \citep{srivastava2014dropout}, and regularization weight.

In the first experiment, we compare a pipeline using the baseline model size with a pipeline using a smaller model size. In the second experiment, we compare a pipeline using Batch Normalization with a pipeline not using Batch Normalization. In the third experiment, we compare a pipeline not using Dropout with a pipeline using Dropout with dropout probability $0.7$. In the fourth experiment, we compare a pipeline not using regularization with a pipeline using L2 regularization with regularization weight $10^{-6}$. 

\begin{table}[!ht]
    \caption{Accuracy of Log Loss Metric (LL) and Calibrated Log Loss Metric (CLL) under Neural Networks (hyperparameters)}
    \centering
        \begin{tabular}{llcc}
        \toprule
            Pipeline A & Pipeline B & LL Acc & CLL Acc \\ 
        \midrule
            Baseline Size & Smaller Size & 69.6\% & 73.6\% \\ 
            BN & no BN & 80.2\% & 89.7\% \\
            no Dropout & p = 0.7 & 95.0\% & 99.3\% \\
            no regularization & weight $10^{-6}$ & 95.2\% & 98.8\% \\
        \bottomrule
            
        \end{tabular}
        \label{exp:hyperparameters}
\end{table}

\begin{figure}[ht]
    \centering
    \subfloat[Log Loss Plot]{
        \includegraphics[width=6.3cm]{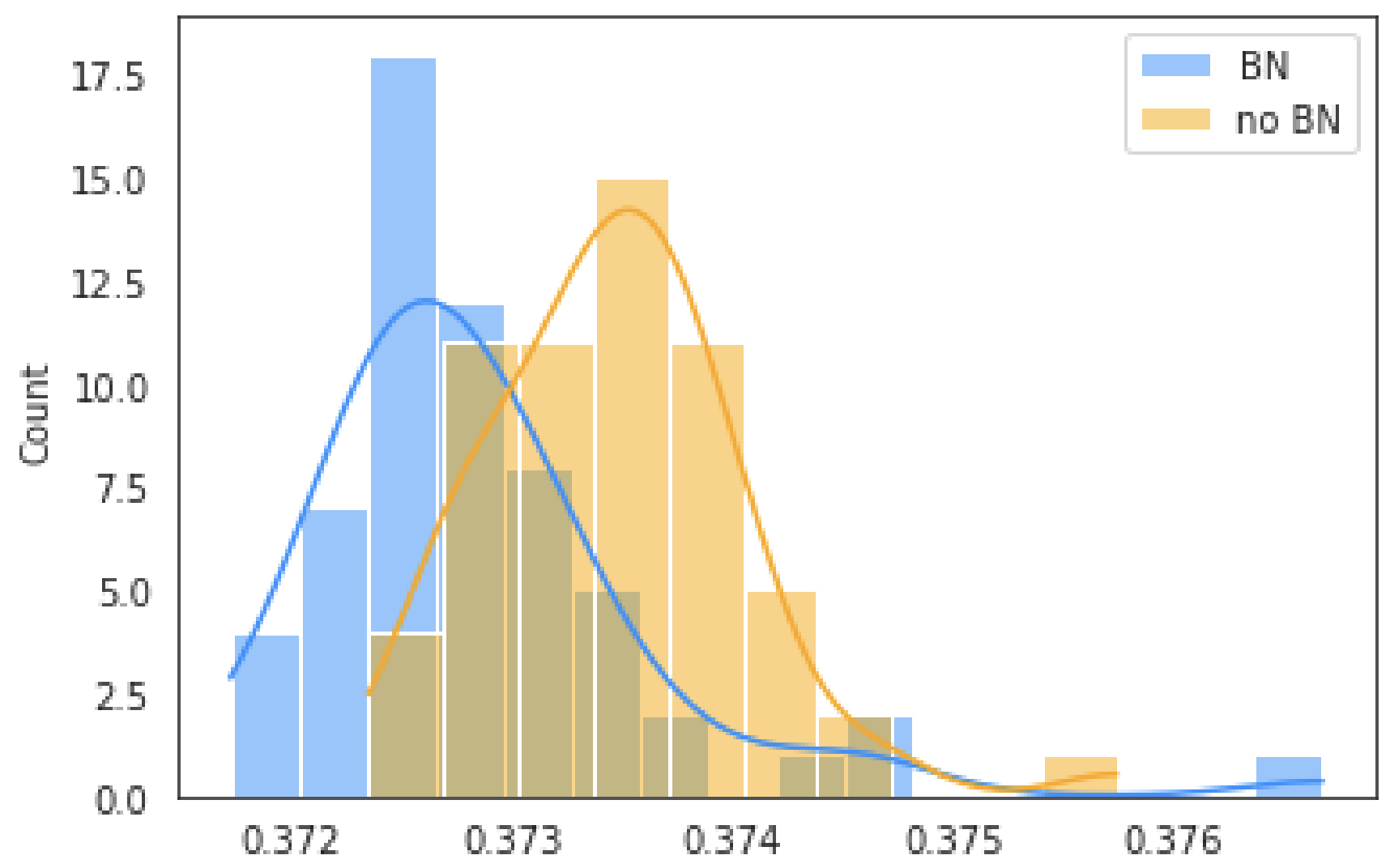}
    }
    \hfill
    \subfloat[Calibrated Log Loss Plot]{
        \includegraphics[width=6.3cm]{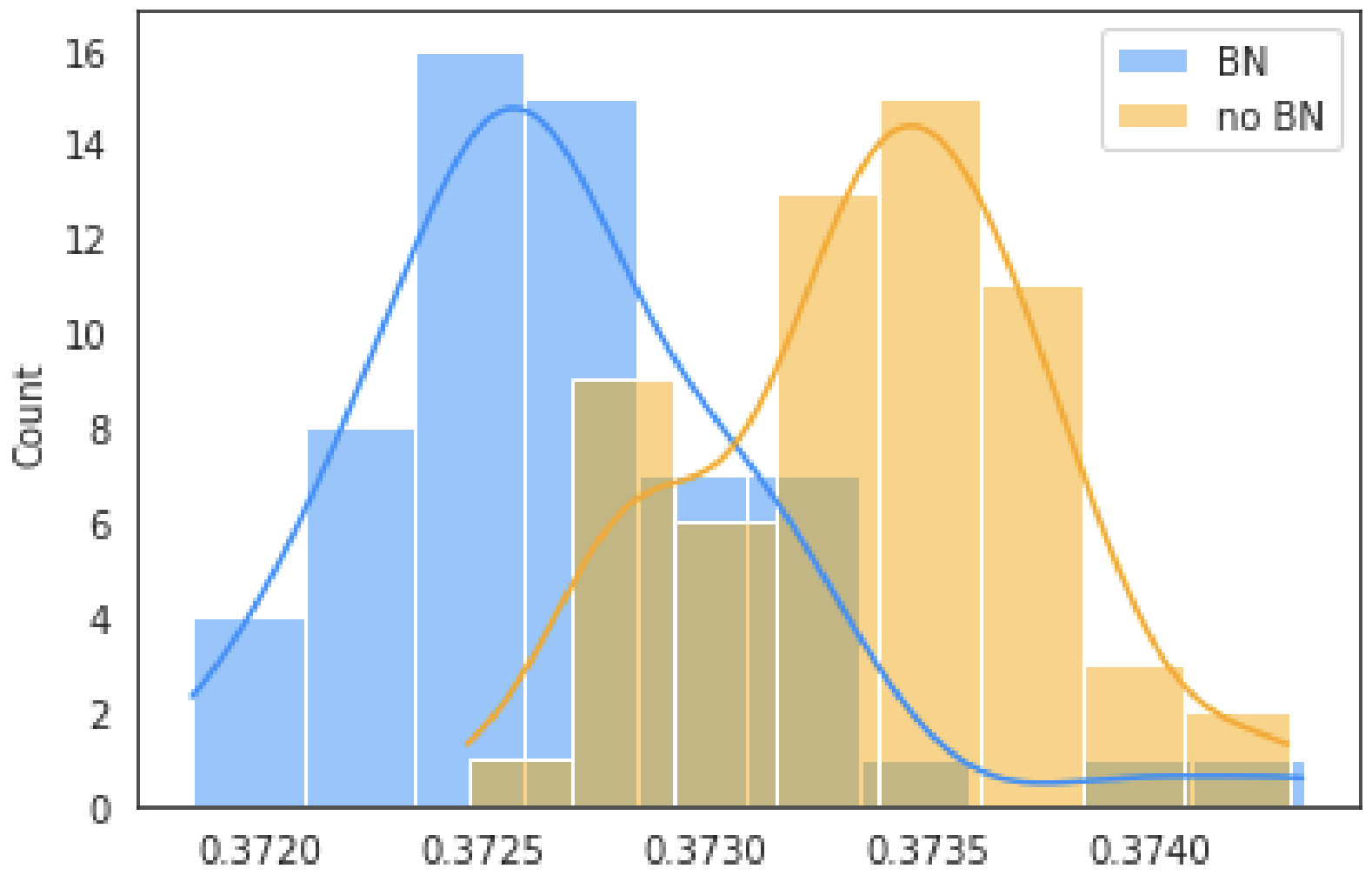}
    }
    \caption{Batch Normalization Experiment}
    \label{fig:bn}
\end{figure}

Figure \ref{fig:bn} illustrates the distribution of Log Loss Metric and Calibration Log Loss Metric observed in the Batch Normalization Experiments. We can clearly see that by using Calibrated Log Loss Metric, it becomes easier to separate pipeline using Batch Normalization from pipeline without Batch Normalization.

From the result in Table \ref{exp:hyperparameters}, we can see that Calibrated Log Loss Metric has a higher accuracy regardless of the hyperparameters we are tuning, indicating its effectiveness when comparing the performance of pipelines with different hyperparameters, which is a very common task in Deep Learning. In Appendix \ref{appendix:avazu} Table \ref{exp:nn_std_hyperparameters}, we report the mean and standard deviation of Log Loss Metric and Calibrated Log Loss Metric, again consistent with previous results.


\textbf{Pipelines with Different Levels of Regularization}
In this set of experiments, we take a closer look at one hyperparameter we conduct in the previous section: regularization weight. For pipeline $A$, we use the baseline model. For pipeline $B$, we use different L2 regularization weights. 

\begin{table}[!ht]
    \centering
    \caption{Accuracy of Log Loss Metric (LL) and Calibrated Log Loss Metric (CLL) under Neural Networks (regularization weight)}
    \begin{tabular}{llcc}
    \toprule     
    Pipeline A & Pipeline B & LL Acc & CLL Acc \\ 
    \midrule
        no regularization & weight $3 \times 10^{-7}$ & 63.2\% & 69.3\% \\
        no regularization & weight $5 \times 10^{-7}$ & 82.2\% & 88.2\% \\
        no regularization & weight $7 \times 10^{-7}$ & 86.6\% & 92.4\% \\
        no regularization & weight $1 \times 10^{-6}$ & 95.2\% & 98.8\% \\
        no regularization & weight $2 \times 10^{-6}$ & 98.8\% & 100.0\% \\
    \bottomrule
    \end{tabular}
    \label{exp:regularization}
\end{table}

From the result in Table \ref{exp:regularization}, we can see that Calibrated Log Loss Metric has a higher accuracy across all different regularization weights, indicating its robustness to different values of regularization weight. As we increase the regularization weight in pipeline $B$, the accuracies of both metrics increase. This is because pipeline $A$ and $B$ differ more with larger regularization weight, making performance comparison easier.

From the result in Appendix \ref{appendix:avazu} Table \ref{exp:nn_std}, we can see that Calibrated Log Loss Metric has a much smaller standard deviation (15\% - 40\% smaller) than Log Loss Metric while the mean of Log Loss Metric and Calibrated Log Loss Metric is almost on par (within 0.05\% difference), again consistent with previous results.

\section{Related Work}

In the realm of recommendation systems, accurately predicting Click-Through Rates (CTR) is paramount for effectively ranking items and achieving business objectives. Before the advent of deep learning technologies, traditional machine learning models such as logistic regression, boosted decision trees, and factorization machines were extensively employed for CTR predictions \citep{friedman2001greedy, koren2009matrix, rendle2010factorization, desrosiers2010comprehensive, canini2012sibyl}. These models laid the groundwork for understanding user interactions and preferences in various contexts.

With the Deep Learning Era, CTR prediction models underwent a significant transformation, embracing deep learning architectures to enhance performance and predictive capabilities \citep{cheng2016wide, guo2017deepfm, covington2016deep, wang2017deep, zhou2018deep, naumov2019deep, lian2018xdeepfm}. These deep models, typically comprising embedding layers followed by multilayer perceptrons (MLP), have set new benchmarks in accurately predicting user clicks and are now commonly employed across industries for personalized recommendation tasks, including ads and content recommendations. Initial deep models like Wide \& Deep \citep{cheng2016wide} and DeepFM \citep{guo2017deepfm} incorporated both wide linear models to memorize feature interactions and deep neural networks to generalize. Further innovations like xDeepFM \citep{lian2018xdeepfm} introduced compressed interaction networks to model high-order feature interactions. The Deep Interest Network \citep{zhou2018deep} also explicitly modeled the hierarchical structure of user interests. Beyond feedforward networks, graph neural networks \citep{gao2023survey} have become the new state-of-the-art approach to recommender systems.

Evaluating the performance of these CTR prediction models is crucial, with several metrics being commonly used for this purpose \citep{yi2013predictive}. The Area Under the ROC Curve (AUC) \citep{fawcett2006introduction, fawcett2004roc} and its variants \cite{zhu2017optimized}, alongside Log Loss, are among the most prevalent metrics in this domain. For example, \citet{he2014practical, wang2017deep, mcmahan2013ad} use Log Loss Metric as their core metric, while \citet{zhou2018deep, mcmahan2013ad} use AUC as their core metric. However, AUC has been criticized for not taking into account the predictive probability \citep{yi2013predictive}. Log Loss, in particular, is favored in scenarios requiring calibrated predictions due to its consideration of predictive probabilities, an aspect crucial for applications like Ads Recommendation Systems \citep{he2014practical}.

Our work contributes a new perspective on CTR evaluation by designing metrics to reduce variance. Most prior innovation has focused on creating new model architectures. By comparison, we aim to improve the evaluation process itself so that modeling advancements can be measured more precisely. Our metric framework is generally applicable to many existing CTR prediction models. By addressing the limitations of existing metrics and leveraging the strengths of deep learning models, our approach aims to refine the evaluation and optimization of recommendation systems further.

\section{Conclusion and Discussion}
\textbf{Conclusion}
In this paper, we have presented a new approach to comparing the performance of different deep learning pipelines. We proposed a new metric framework, Calibrated Loss Metric, which has a higher accuracy and smaller variance than its vanilla counterpart for a wide range of pipelines. Our experiments in section \ref{exp} demonstrated the superiority of Calibrated Loss Metric, and we believe this new metric can be used to more effectively and efficiently compare the performance of different pipelines in similar settings. Future work includes expanding this idea to evaluate NLP pipelines, and establish theoretical guarantees under more general settings.

\textbf{Limitations}
Our method sacrifices accuracy when comparing some specific pipelines. For example, if pipeline $B$ can reliably improve the model calibration in test distribution over pipeline $A$, Calibrated Log Loss Metric will not be able to correctly detect the benefits of pipeline B, while Log Loss Metric is able to. 
However, for most pipeline comparisons conducted in industry and academia like feature engineering, tuning parameters, etc., Calibrated Log Loss Metric has a huge accuracy boost over Log Loss Metric as we demonstrated in Section \ref{exp}.

\textbf{Potential Applications}
Our method may have applications in AutoML domain. AutoML (Automated Machine Learning) systems are designed to automate the process of selecting, designing, and tuning machine learning models, and a key component of these systems is the selection of the best-performing pipeline (e.g. hyperparameters, model architectures etc.). The new metric can be used as a more accurate way of comparing the performance and selecting the best one. The new metric is in particular useful when performing hyperparameter tuning.


\newpage
\appendix

\section{Proofs}
\label{appendix:proof}

\subsection{Proofs of Theorem \ref{theorem} and Corollary \ref{corollary}
}

\begin{lemma}
Suppose that $\hat{\beta}_{n}$ is the unique linear regression solution 
computed using the training data $\left\{ X_{i},Y_{i}\right\} _{i=1}^{n}.$ 
Then, $\hat{\beta}_{n}$ is independent to $\left\{ \bar{X},\bar{Y}\right\} ,$
where  
\begin{equation*}
\bar{X}=\frac{1}{n}\sum_{i=1}^{n}X_{i},\text{ and }\bar{Y}=\frac{1}{n}
\sum_{i=1}^{n}Y_{i}.
\end{equation*}
\label{lemma:independence}
\end{lemma}

\begin{proof}
It is well-known that $\hat{\beta}_{n}$ is the solution of the convex 
program  
\begin{equation*}
\min_{\beta ,c}\sum_{i=1}^{n}\left( Y_{i}-\beta ^{\top }X_{i}-c\right) ^{2},
\end{equation*}
which is equivalent to the convex program  
\begin{eqnarray}
&&\min_{\beta }\sum_{i=1}^{n}\left( Y_{i}-\beta ^{\top }X_{i}-\left( \bar{Y}
-\beta ^{\top }\bar{X}\right) \right) ^{2}  \label{convex:demean} \\
&=&\min_{\beta }\sum_{i=1}^{n}\left( \left( Y_{i}-\bar{Y}\right) -\beta
^{\top }\left( X_{i}-\bar{X}\right) \right) ^{2}.  \notag
\end{eqnarray}
Let $\tilde{Y}_{i}=Y_{i}-\bar{Y}$ and $\tilde{X}_{i}=Y_{i}-\bar{Y}.$ Note 
that $\tilde{Y}_{i}$ is independent to $\bar{Y}$ and $\tilde{X}_{i}$ is 
independent to $\bar{X}$ as  
\begin{equation*}
cov(Y_{i}-\bar{Y},\bar{Y})=0,cov(X_{i}-\bar{X},\bar{X})=0,
\end{equation*}
and $\left\{ X,Y\right\} $ are jointly normal. Note that the convex program
(\ref{convex:demean}) yields that $\hat{\beta}_{n}$ is a function of $%
\left\{  \tilde{X}_{i},\tilde{Y}_{i}\right\} _{i=1}^{n},$ which is
independent to $ \left\{ \bar{X},\bar{Y}\right\} .$
\end{proof}

\begin{theorem}
Suppose that the features $X\in \mathbb{R}^d$ and the label $Y$ are
distributed jointly Gaussian. We consider linear regression $h(x) = \beta
^\top x +\alpha $. Let $\hat{\beta}_n$ be the coefficient learned from the
training data with sample size $n$. Then, we have  
\begin{equation*}
\left(1+\frac{1}{n}\right)\mathbb{E}[e_1(h(X),Y)|\hat{\beta}_n ]= \mathbb{E}%
[e(h(X),Y)|\hat{\beta}_n],
\end{equation*}
where the expectation is taken over the randomness over both the training
and test samples.
\end{theorem}

\begin{proof}
Note that the learned bias $\hat{\alpha}=\bar{Y}-\hat{\beta}_{n}^{\intercal }%
\bar{X},$ where $\bar{Y}$ and $\bar{X}$ are the empirical average of the
samples in the training set.  Then, The risks are defined as 
\begin{eqnarray*}
\mathbb{E}\left[ e(h(X),Y)|\hat{\beta}_{n}\right]  &=&\mathbb{E}\left[
\left( \left( Y-\bar{Y}\right) -\hat{\beta}_{n}^{\intercal }\left( X-\bar{X}%
\right) \right) ^{2}|\hat{\beta}_{n}\right] , \\
\mathbb{E}\left[ e_{1}(h(X),Y)|\hat{\beta}_{n}\right]  &=&\mathbb{E}\left[
\left( \left( Y-\mathbb{E}_{\mathcal{D}}[Y]\right) -\hat{\beta}%
_{n}^{\intercal }\left( X-\mathbb{E}_{\mathcal{D}}[X]\right) \right) ^{2}|%
\hat{\beta}_{n}\right] .
\end{eqnarray*}%
Therefore, we have 
\begin{equation*}
\mathbb{E}\left[ e_{1}(h(X),Y)|\hat{\beta}_{n}\right] =var(Y-\hat{\beta}%
_{n}^{\intercal }X|\hat{\beta}_{n}).
\end{equation*}%
Note that we have 
\begin{equation*}
\left\{ Y-\bar{Y},X-\bar{X}\right\} \overset{d}{=}\sqrt{1+\frac{1}{n}}%
\left\{ Y-\mathbb{E}_{\mathcal{D}}[Y],X-\mathbb{E}_{\mathcal{D}}[X]\right\} .
\end{equation*}%
given that $\left\{ Y,X\right\}$ is independent to $\left\{ \bar{Y},\bar{X}\right\}$. 
Recall for \ref{lemma:independence} that $\hat{\beta}_{n}$ is independent to $\left\{ \bar{Y},\bar{X}%
\right\} ,$ we have
\begin{eqnarray*}
\mathbb{E}\left[ e(h(X),Y)|\hat{\beta}_{n}\right]  &=&var\left( \left( Y-%
\bar{Y}\right) -\hat{\beta}_{n}^{\intercal }\left( X-\bar{X}\right) |\hat{%
\beta}_{n}\right)  \\
&=&\left( 1+\frac{1}{n}\right) var(Y-\hat{\beta}_{n}^{\intercal }X|\hat{\beta%
}_{n}).
\end{eqnarray*}
\end{proof}

\begin{corollary}
Suppose that $h_1(x)$ and $h_2(x)$ are two different learned linear
functions in different feature sets. Then, we have  
\begin{eqnarray}
\mathbb{E}[R_e(h_1)] =\mathbb{E}[R_e(h_2)] \Leftrightarrow \mathbb{E}%
[R_{e_1}(h_1)] =\mathbb{E}[R_{e_1}(h_2)]
\end{eqnarray}
and  $var(\left(1+{1}/{n}\right)R_{e_1}(h))<var(R_{e}(h))$ for any $h$
learned from linear regression.
\end{corollary}

\begin{proof}
From the definition, we see 
\begin{eqnarray*}
\mathbb{E}[R_{e}(h)] &=&\mathbb{E}\left[ \mathbb{E}\left[ e(h(X),Y)|\hat{%
\beta}_{n}\right] \right] , \\
\mathbb{E}[R_{e_{1}}(h)] &=&\mathbb{E}\left[ \mathbb{E}\left[ e_{1}(h(X),Y)|%
\hat{\beta}_{n}\right] \right] .
\end{eqnarray*}%
Therefore, we conclude the first claim.

For the second claim, note that 
\begin{eqnarray*}
R_{e_{1}}(h) &=&\mathbb{E}\left[ e_{1}(h(X),Y)|\hat{\beta}_{n}\right] , \\
R_{e}(h) &=&\mathbb{E}\left[ e_{1}(h(X),Y)|\hat{\beta}_{n},\bar{X},\bar{Y}%
\right] .
\end{eqnarray*}%
Then, the variance of $R_{e}(h)$ can be decomposed as 

\begin{align*}
var(R_{e}(h)) &= var\left( \mathbb{E}\left[ e(h(X),Y)|\hat{\beta}_{n}\right] \right) \\
&\quad +\mathbb{E}\left[ var\left( \mathbb{E}\left[ e(h(X),Y)|\hat{\beta}_{n}^{\intercal },\bar{X},\bar{Y}\right] |\hat{\beta}_{n}\right) \right]  \\
&> var\left( \mathbb{E}\left[ e(h(X),Y)|\hat{\beta}_{n}\right] \right)  \\
&= var\left( \left( 1+\frac{1}{n}\right) \mathbb{E}\left[ e_{1}(h(X),Y)|\hat{\beta}_{n}\right] \right)  \\
&= var\left( \left( 1+\frac{1}{n}\right) R_{e_{1}}(h)\right) .
\end{align*}
\end{proof}

\section{Experiments}
\label{appendix:experiments}

\subsection{Synthetic Data: Linear Regression}

We consider a linear regression model in this section to give empirical evidence to support our theory. We assume the response $Y$ follows the following generating process:
\begin{equation}
    Y = \beta^\top X + \epsilon,
\end{equation}
where $\epsilon \sim \mathcal{N}(\mu_{\mathrm{e}},\Sigma_{\mathrm{e}})$ and $\beta,X\in \mathbb{R}^d$.

In the experiments, we consider $d=20$, $\beta = [1,1,\ldots,1]^\top$,  and $X\sim\mathcal{N}(\mu_{\mathcal{D}},\Sigma_{\mathcal{D}})$ in both the training set and the test set.  In the training set, we generate $N_{\mathrm{train}}=1000$ i.i.d. training samples to train a linear regression model. In the test set, we generate $N_{\mathrm{test}}=11000$ i.i.d. test samples, with $N_{\mathrm{val-test}}=1000$ and $N_{\mathrm{remaining-test}}=10000$.

We assume
$\mu_{\mathcal{D}} =
  [-0.05,-0.05,\ldots,-0.05]^\top,
 $
  $\Sigma_{\mathcal{D}} =0.25^2\times I_{d\times d}$, $\mu_{\mathrm{e}} = 1$ and $\Sigma_{\mathrm{e}} = 2$.

Note that there is no randomness in the training process of Linear Regression, as it's a convex optimization program. The randomness of Linear Regression comes from the training data. In order to run pipelines $A$ and $B$ multiple times to estimate the metric accuracy, we re-sample training data each time from the ground truth data distribution. 

For pipeline $A$, we use all the $20$ features available, and for pipeline $B$, we use the first $19$ features and leave the last feature out. It's clear that pipeline $A$ should perform better than pipeline $B$ in the ground truth.

For each round of experiments, we run pipelines A and B for $100$ times and report accuracy $\mathrm{\widehat{Acc}}(\bar{e})$ in the table \ref{exp:linear_regression}. We performed $20$ rounds of experiments, and report the mean and the standard errors of $\mathrm{\widehat{Acc}}(\bar{e})$ in Table \ref{exp:linear_regression}. We also calculate the standard deviation and mean of Quadratic Loss~(QL) Metric and Calibrated Quadratic Loss~(CQL) Metric from pipeline A in each round of experiments, and report the average in Table \ref{exp:linear_regression_std}.
\begin{table}[h]
    \centering
    \caption{Accuracy of Quadratic Loss Metric (QL) and Calibrated Quadratic Loss (CQL) Metric under Linear Regression}
        \begin{tabular}{cc|cc}
        \toprule
            \multicolumn{2}{c|}{\# of feature } & \multicolumn{2}{c}{Accuracy } \\
        \midrule
            Pipeline A & Pipeline B & QL & CQL \\
        \midrule
            20 & 19 & 93.49\% $\pm$ 0.35\% & 95.81\% $\pm$ 0.28\% \\
        \bottomrule
        \end{tabular}
    \label{exp:linear_regression}
\end{table}

\begin{table}[h]
    \centering
    \caption{Mean and Standard Deviation of Quadratic Loss Metric (QL) and Calibrated Quadratic Loss (CQL) Metric under Linear Regression}
        \begin{tabular}{cccc}
        \toprule
            QL Mean & CQL Mean & QL Std & CQL Std \\
        \midrule
            4.067 & 4.070 & 0.0295 & 0.0286 \\
        \bottomrule
        \end{tabular}
    \label{exp:linear_regression_std}
\end{table}

From the result in Table \ref{exp:linear_regression}, we can see that Calibrated Quadratic Loss Metric has a higher accuracy compared with Quadratic Loss Metric. From the result in Table \ref{exp:linear_regression_std}, we can see that Calibrated Quadratic Loss Metric indeed has a smaller standard deviation (3.1\% smaller) than Quadratic Loss Metric while the mean of Quadratic Loss Metric and Calibrated Quadratic Loss Metric is almost on par (0.07\% difference).

\subsection{Synthetic Data: Logistic Regression}

We consider a logistic regression model. We assume the response $Y$ follows the Bernoulli distribution with probability $\left(1+\exp(-\beta^\top X)\right)^{-1}$, for $\beta,X\in \mathbb{R}^d$.

In the experiments, we consider $d=20$, $\beta = [1,1,\ldots,1]^\top$, and $X\sim\mathcal{N}(\mu_{\mathcal{D}},\Sigma_{\mathcal{D}})$ in both the training and test sets.  In the training set, we generate $N_{\mathrm{train}}=1000$ i.i.d. training samples to train a logistic regression model. In the test set, we generate $N_{\mathrm{test}}=12000$ i.i.d. test samples, with $N_{\mathrm{val-test}}=2000$ and $N_{\mathrm{remaining-test}}=10000$.

We assume
  $\mu_{\mathcal{D}} =
  [-0.05,-0.05,\ldots,-0.05]^\top
$
  and $\Sigma_{\mathcal{D}}=0.25^2\times I_{d\times d}$.

Note that similar to Linear Regression, there is no randomness in the training process of Logistic Regression as well, as it's a convex optimization program. The randomness of Logistic Regression comes from the training data. We employ the same strategy to estimate the metric accuracy, i.e. we re-sample training data each time from the ground truth data distribution. 

For pipeline $A$, we use all the $20$ available features, and for pipeline $B$, we use the first $19$ features and leave the last feature out. It's clear that pipeline $A$ should perform better than pipeline $B$ in the ground truth.

For each round of experiments, we run pipelines A and B for $1000$ times and report accuracy $\mathrm{\widehat{Acc}}(\bar{e})$ in the table \ref{exp:lr}. We performed $20$ rounds of experiments, and report the mean and the standard errors of $\mathrm{\widehat{Acc}}(\bar{e})$ in Table \ref{exp:lr}. We also calculate the standard deviation and mean of Log Loss~(LL) Metric and Calibrated Log Loss~(CLL) Metric from pipeline A in each round of experiments, and report the average in Table \ref{exp:lr_std}.

\begin{table}[h]
    \centering
    \caption{Accuracy of Log Loss Metric (LL) and Calibrated Log Loss Metric (CLL) under Logistic Regression}
        \begin{tabular}{cc|cc}
        \toprule
            \multicolumn{2}{c|}{\# of feature } & \multicolumn{2}{c}{Accuracy } \\
        \midrule
            Pipeline A & Pipeline B & LL & CLL \\
        \midrule
            20 & 19 & 79.62\% $\pm$ 0.18\% & 83.7\% $\pm$ 0.15\% \\
        \bottomrule
        \end{tabular}
    \label{exp:lr}
\end{table}

\begin{table}[h]
    \centering
    \caption{Mean and Standard Deviation of Log Loss Metric (LL) and Calibrated Log Loss Metric (CLL) under Logistic Regression}
        \begin{tabular}{cccc}
        \toprule
            LL Mean & CLL Mean & LL Std & CLL Std \\
        \midrule
            0.5366 & 0.5343 & 0.00385 & 0.00370 \\
        \bottomrule
        \end{tabular}
    \label{exp:lr_std}
\end{table}

From the result in Table \ref{exp:lr}, we can clearly see that Calibrated Log Loss Metric has a huge accuracy boost compared with Log Loss Metric. From the result in Table \ref{exp:lr_std}, we can see that Calibrated Log Loss Metric indeed has a smaller standard deviation (3.9\% smaller) than Log Loss Metric while the mean of Log Loss Metric and Calibrated Log Loss Metric is almost on par (0.43\% difference).

\subsection{Avazu CTR Prediction dataset}
\label{appendix:avazu}
We report the mean and standard deviation of Log Loss and Calibrated Log Loss for additional Avazu CTR Prediction dataset experiments. 

\begin{table}[!ht]
    \centering
    \caption{Mean and Standard Deviation of Log Loss Metric (LL) and Calibrated Log Loss Metric (CLL) under Neural Networks (hyperparameters)}
        \begin{tabular}{lcccc}
        \toprule
            Pipeline & LL Mean & CLL Mean & LL Std & CLL Std \\
        \midrule
            Baseline & 0.37347 & 0.37338 & $5.8 \times 10^{-4}$ & $3.8 \times 10^{-4}$ \\
            smaller size & 0.37383 & 0.37374 & $4.8 \times 10^{-4}$ & $3.9 \times 10^{-4}$ \\
            BN & 0.37286 & 0.37268 & $7.8 \times 10^{-4}$ & $4.4 \times 10^{-4}$ \\ 
            Dropout & 0.37454 & 0.37456 & $3.9 \times 10^{-4}$ & $3.6 \times 10^{-4}$ \\
            Regularization & 0.37475 & 0.37459 & $6.1 \times 10^{-4}$ & $4.2 \times 10^{-4}$ \\
        \bottomrule
        \end{tabular}
    \label{exp:nn_std_hyperparameters}
\end{table}

\begin{table}[!ht]
    \centering
    \caption{Mean and Standard Deviation of Log Loss Metric (LL) and Calibrated Log Loss Metric (CLL) under Neural Networks (regularization weight)}
        \begin{tabular}{lcccc}
        \toprule
            Pipeline & LL Mean & CLL Mean & LL Std & CLL Std \\
        \midrule
            0 & 0.37347 & 0.37338 & $5.8 \times 10^{-4}$ & $3.8 \times 10^{-4}$ \\
            $3 \times 10^{-7}$ & 0.37371 & 0.37369 & $4.8 \times 10^{-4}$ & $4.3 \times 10^{-4}$ \\
            $5 \times 10^{-7}$ & 0.37419 & 0.37411 & $5.9 \times 10^{-4}$ & $5.0 \times 10^{-4}$ \\ 
            $7 \times 10^{-7}$ & 0.37428 & 0.37421 & $5.7 \times 10^{-4}$ & $4.4 \times 10^{-4}$ \\ 
            $1 \times 10^{-6}$ & 0.37475 & 0.37459 & $6.1 \times 10^{-4}$ & $4.2 \times 10^{-4}$ \\ 
            $2 \times 10^{-6}$ & 0.37562 & 0.37547 & $5.9 \times 10^{-4}$ & $4.2 \times 10^{-4}$ \\ 
        \bottomrule
        \end{tabular}
    \label{exp:nn_std}
\end{table}

\bibliography{references}

\begin{thebibliography}{28}
\providecommand{\natexlab}[1]{#1}
\providecommand{\url}[1]{\texttt{#1}}
\expandafter\ifx\csname urlstyle\endcsname\relax
  \providecommand{\doi}[1]{doi: #1}\else
  \providecommand{\doi}{doi: \begingroup \urlstyle{rm}\Url}\fi

\bibitem[Bouthillier and Varoquaux(2020)]{bouthillier2020survey}
Xavier Bouthillier and Ga{\"e}l Varoquaux.
\newblock \emph{Survey of machine-learning experimental methods at NeurIPS2019 and ICLR2020}.
\newblock PhD thesis, Inria Saclay Ile de France, 2020.

\bibitem[Bouthillier et~al.(2021)Bouthillier, Delaunay, Bronzi, Trofimov, Nichyporuk, Szeto, Mohammadi~Sepahvand, Raff, Madan, Voleti, et~al.]{bouthillier2021accounting}
Xavier Bouthillier, Pierre Delaunay, Mirko Bronzi, Assya Trofimov, Brennan Nichyporuk, Justin Szeto, Nazanin Mohammadi~Sepahvand, Edward Raff, Kanika Madan, Vikram Voleti, et~al.
\newblock Accounting for variance in machine learning benchmarks.
\newblock \emph{Proceedings of Machine Learning and Systems}, 3:\penalty0 747--769, 2021.

\bibitem[Canini et~al.(2012)Canini, Chandra, Ie, McFadden, Goldman, Gunter, Harmsen, LeFevre, Lepikhin, Llinares, et~al.]{canini2012sibyl}
Kevin Canini, Tushar Chandra, Eugene Ie, Jim McFadden, Ken Goldman, Mike Gunter, Jeremiah Harmsen, Kristen LeFevre, Dmitry Lepikhin, Tomas~Lloret Llinares, et~al.
\newblock Sibyl: A system for large scale supervised machine learning.
\newblock \emph{Technical Talk}, 1:\penalty0 113, 2012.

\bibitem[Cheng et~al.(2016)Cheng, Koc, Harmsen, Shaked, Chandra, Aradhye, Anderson, Corrado, Chai, Ispir, et~al.]{cheng2016wide}
Heng-Tze Cheng, Levent Koc, Jeremiah Harmsen, Tal Shaked, Tushar Chandra, Hrishi Aradhye, Glen Anderson, Greg Corrado, Wei Chai, Mustafa Ispir, et~al.
\newblock Wide \& deep learning for recommender systems.
\newblock In \emph{Proceedings of the 1st workshop on deep learning for recommender systems}, pages 7--10, 2016.

\bibitem[Covington et~al.(2016)Covington, Adams, and Sargin]{covington2016deep}
Paul Covington, Jay Adams, and Emre Sargin.
\newblock Deep neural networks for youtube recommendations.
\newblock In \emph{Proceedings of the 10th ACM conference on recommender systems}, pages 191--198, 2016.

\bibitem[Desrosiers and Karypis(2010)]{desrosiers2010comprehensive}
Christian Desrosiers and George Karypis.
\newblock A comprehensive survey of neighborhood-based recommendation methods.
\newblock \emph{Recommender systems handbook}, pages 107--144, 2010.

\bibitem[Fawcett(2004)]{fawcett2004roc}
Tom Fawcett.
\newblock Roc graphs: Notes and practical considerations for researchers.
\newblock \emph{Machine learning}, 31\penalty0 (1):\penalty0 1--38, 2004.

\bibitem[Fawcett(2006)]{fawcett2006introduction}
Tom Fawcett.
\newblock An introduction to roc analysis.
\newblock \emph{Pattern recognition letters}, 27\penalty0 (8):\penalty0 861--874, 2006.

\bibitem[Friedman(2001)]{friedman2001greedy}
Jerome~H Friedman.
\newblock Greedy function approximation: a gradient boosting machine.
\newblock \emph{Annals of statistics}, pages 1189--1232, 2001.

\bibitem[Gao et~al.(2023)Gao, Zheng, Li, Li, Qin, Piao, Quan, Chang, Jin, He, et~al.]{gao2023survey}
Chen Gao, Yu~Zheng, Nian Li, Yinfeng Li, Yingrong Qin, Jinghua Piao, Yuhan Quan, Jianxin Chang, Depeng Jin, Xiangnan He, et~al.
\newblock A survey of graph neural networks for recommender systems: Challenges, methods, and directions.
\newblock \emph{ACM Transactions on Recommender Systems}, 1\penalty0 (1):\penalty0 1--51, 2023.

\bibitem[Guo et~al.(2017)Guo, Tang, Ye, Li, and He]{guo2017deepfm}
Huifeng Guo, Ruiming Tang, Yunming Ye, Zhenguo Li, and Xiuqiang He.
\newblock Deepfm: a factorization-machine based neural network for ctr prediction.
\newblock \emph{arXiv preprint arXiv:1703.04247}, 2017.

\bibitem[He et~al.(2014)He, Pan, Jin, Xu, Liu, Xu, Shi, Atallah, Herbrich, Bowers, et~al.]{he2014practical}
Xinran He, Junfeng Pan, Ou~Jin, Tianbing Xu, Bo~Liu, Tao Xu, Yanxin Shi, Antoine Atallah, Ralf Herbrich, Stuart Bowers, et~al.
\newblock Practical lessons from predicting clicks on ads at facebook.
\newblock In \emph{Proceedings of the Eighth International Workshop on Data Mining for Online Advertising}, pages 1--9, 2014.

\bibitem[Ioffe and Szegedy(2015)]{ioffe2015batch}
Sergey Ioffe and Christian Szegedy.
\newblock Batch normalization: Accelerating deep network training by reducing internal covariate shift.
\newblock In \emph{International conference on machine learning}, pages 448--456. pmlr, 2015.

\bibitem[Koren et~al.(2009)Koren, Bell, and Volinsky]{koren2009matrix}
Yehuda Koren, Robert Bell, and Chris Volinsky.
\newblock Matrix factorization techniques for recommender systems.
\newblock \emph{Computer}, 42\penalty0 (8):\penalty0 30--37, 2009.

\bibitem[Lian et~al.(2018)Lian, Zhou, Zhang, Chen, Xie, and Sun]{lian2018xdeepfm}
Jianxun Lian, Xiaohuan Zhou, Fuzheng Zhang, Zhongxia Chen, Xing Xie, and Guangzhong Sun.
\newblock xdeepfm: Combining explicit and implicit feature interactions for recommender systems.
\newblock In \emph{Proceedings of the 24th ACM SIGKDD international conference on knowledge discovery \& data mining}, pages 1754--1763, 2018.

\bibitem[McMahan et~al.(2013)McMahan, Holt, Sculley, Young, Ebner, Grady, Nie, Phillips, Davydov, Golovin, et~al.]{mcmahan2013ad}
H~Brendan McMahan, Gary Holt, David Sculley, Michael Young, Dietmar Ebner, Julian Grady, Lan Nie, Todd Phillips, Eugene Davydov, Daniel Golovin, et~al.
\newblock Ad click prediction: a view from the trenches.
\newblock In \emph{Proceedings of the 19th ACM SIGKDD international conference on Knowledge discovery and data mining}, pages 1222--1230, 2013.

\bibitem[Naumov et~al.(2019)Naumov, Mudigere, Shi, Huang, Sundaraman, Park, Wang, Gupta, Wu, Azzolini, et~al.]{naumov2019deep}
Maxim Naumov, Dheevatsa Mudigere, Hao-Jun~Michael Shi, Jianyu Huang, Narayanan Sundaraman, Jongsoo Park, Xiaodong Wang, Udit Gupta, Carole-Jean Wu, Alisson~G Azzolini, et~al.
\newblock Deep learning recommendation model for personalization and recommendation systems.
\newblock \emph{arXiv preprint arXiv:1906.00091}, 2019.

\bibitem[Picard(2021)]{seedallyouneed}
David Picard.
\newblock Torch.manual{\_}seed(3407) is all you need: On the influence of random seeds in deep learning architectures for computer vision.
\newblock \emph{CoRR}, abs/2109.08203, 2021.
\newblock URL \url{https://arxiv.org/abs/2109.08203}.

\bibitem[Rendle(2010)]{rendle2010factorization}
Steffen Rendle.
\newblock Factorization machines.
\newblock In \emph{2010 IEEE International conference on data mining}, pages 995--1000. IEEE, 2010.

\bibitem[Shen(2017)]{shen2017deepctr}
Weichen Shen.
\newblock Deepctr: Easy-to-use,modular and extendible package of deep-learning based ctr models.
\newblock \url{https://github.com/shenweichen/deepctr}, 2017.

\bibitem[Srivastava et~al.(2014)Srivastava, Hinton, Krizhevsky, Sutskever, and Salakhutdinov]{srivastava2014dropout}
Nitish Srivastava, Geoffrey Hinton, Alex Krizhevsky, Ilya Sutskever, and Ruslan Salakhutdinov.
\newblock Dropout: a simple way to prevent neural networks from overfitting.
\newblock \emph{The journal of machine learning research}, 15\penalty0 (1):\penalty0 1929--1958, 2014.

\bibitem[Wang et~al.(2017)Wang, Fu, Fu, and Wang]{wang2017deep}
Ruoxi Wang, Bin Fu, Gang Fu, and Mingliang Wang.
\newblock Deep \& cross network for ad click predictions.
\newblock In \emph{Proceedings of the ADKDD'17}, pages 1--7. 2017.

\bibitem[Wang et~al.(2021)Wang, Shivanna, Cheng, Jain, Lin, Hong, and Chi]{wang2021dcn}
Ruoxi Wang, Rakesh Shivanna, Derek Cheng, Sagar Jain, Dong Lin, Lichan Hong, and Ed~Chi.
\newblock Dcn v2: Improved deep \& cross network and practical lessons for web-scale learning to rank systems.
\newblock In \emph{Proceedings of the Web Conference 2021}, pages 1785--1797, 2021.

\bibitem[Yi et~al.(2013)Yi, Chen, Li, Sett, and Yan]{yi2013predictive}
Jeonghee Yi, Ye~Chen, Jie Li, Swaraj Sett, and Tak~W Yan.
\newblock Predictive model performance: Offline and online evaluations.
\newblock In \emph{Proceedings of the 19th ACM SIGKDD international conference on Knowledge discovery and data mining}, pages 1294--1302, 2013.

\bibitem[Zhang et~al.(2016)Zhang, Du, and Wang]{zhang2016deep}
Weinan Zhang, Tianming Du, and Jun Wang.
\newblock Deep learning over multi-field categorical data.
\newblock In \emph{European conference on information retrieval}, pages 45--57. Springer, 2016.

\bibitem[Zhang et~al.(2022)Zhang, Sheng, Zhang, Jiang, Han, Deng, and Zheng]{zhang2022towards}
Zhao-Yu Zhang, Xiang-Rong Sheng, Yujing Zhang, Biye Jiang, Shuguang Han, Hongbo Deng, and Bo~Zheng.
\newblock Towards understanding the overfitting phenomenon of deep click-through rate prediction models.
\newblock \emph{arXiv preprint arXiv:2209.06053}, 2022.

\bibitem[Zhou et~al.(2018)Zhou, Zhu, Song, Fan, Zhu, Ma, Yan, Jin, Li, and Gai]{zhou2018deep}
Guorui Zhou, Xiaoqiang Zhu, Chenru Song, Ying Fan, Han Zhu, Xiao Ma, Yanghui Yan, Junqi Jin, Han Li, and Kun Gai.
\newblock Deep interest network for click-through rate prediction.
\newblock In \emph{Proceedings of the 24th ACM SIGKDD International Conference on Knowledge Discovery \& Data Mining}, pages 1059--1068, 2018.

\bibitem[Zhu et~al.(2017)Zhu, Jin, Tan, Pan, Zeng, Li, and Gai]{zhu2017optimized}
Han Zhu, Junqi Jin, Chang Tan, Fei Pan, Yifan Zeng, Han Li, and Kun Gai.
\newblock Optimized cost per click in taobao display advertising.
\newblock In \emph{Proceedings of the 23rd ACM SIGKDD international conference on knowledge discovery and data mining}, pages 2191--2200, 2017.

\end{thebibliography}

\end{document}